\documentclass[journal]{IEEEtran}

\usepackage{amsmath,amssymb,amsthm}
\usepackage{booktabs,array,multirow,tabularx}
\usepackage{siunitx}
\usepackage{graphicx}
\usepackage{xcolor}
\usepackage{microtype}
\usepackage{cite}
\usepackage{url}
\usepackage{algorithm}
\usepackage{algpseudocode}
\usepackage{balance}

\newcommand{\rams}{RAMS}
\newcommand{\swas}{SWAS}

\theoremstyle{plain}
\newtheorem{proposition}{Proposition}

\newtheorem{corollary}{Corollary}
\theoremstyle{definition}
\newtheorem{definition}{Definition}
\theoremstyle{remark}
\newtheorem{remark}{Remark}

\algrenewcommand\algorithmicrequire{\textbf{Input:}}
\algrenewcommand\algorithmicensure{\textbf{Output:}}

\begin{document}

\title{RAMS: Resource-Adaptive and Detection-Conditioned Model Switching
for Embedded Edge Perception}

\author{
Kushal~Khemani$^{1}$ \and
Evan~Leri$^{1}$ \and
George~Xu$^{1}$ \and
Amit~Hod$^{1}$ \\[4pt]
$^{1}$NEXEDGE Research Lab
\thanks{Corresponding author: kushal.khemani@gmail.com}
\thanks{Evan Leri: evan@jadewire.dev}
\thanks{George Xu: georgexu5588@gmail.com}
\thanks{Amit Hod: amityhod@gmail.com}
}

\maketitle

\begin{abstract}
Edge object detection requires balancing latency and detection quality under
changing resource pressure. This paper presents \rams{}, a lightweight runtime
controller that monitors device pressure, calibrates switching thresholds from
idle behavior, and selects among three resident YOLOv8 tiers
(NANO/SMALL/MEDIUM at 320/416/640\,px) without model-reload latency. The
controller defines five switching policies including two detection-conditioned
variants that prevent aggressive downgrades after recent vulnerable-road-user
(VRU) detections. We also define the VRU-Weighted Accuracy Score (\swas{}),
a scalar metric for offline policy comparison without ground-truth annotations,
together with an oracle-bounded variant that separates detector circularity
from true tier-retention benefit. Across Raspberry Pi~5, x86 laptops, and
Jetson Orin ONNX/TensorRT deployments, the same controller equations operate
over a 37$\times$ latency range. On Jetson Orin TensorRT under heavy load,
\texttt{safety2} achieves \SI{3.41}{ms} mean latency, 5.6$\times$ faster than
fixed-MEDIUM, while retaining 74\% of its proxy accuracy via near-NANO
operation with selective SMALL and MEDIUM locks during VRU-positive windows.
Detection-conditioned switching improves \swas{} by 25.4\% under oracle
scoring and 47.3\% under detector-derived scoring over threshold-only policies
under heavy load; the oracle figure is the conservative estimate that accounts
for detector circularity. Live KITTI evaluation reports per-tier VRU recall of
24.2\%, 41.2\%, and 59.0\%, bounding reactive-override effectiveness: the
detection-conditioned lock cannot fire on the 76\% of VRU-present frames that
NANO fails to detect.
\end{abstract}


\section{Introduction}
\label{sec:intro}

Deploying object detection on embedded hardware involves a tension between
inference latency and detection quality that does not arise in cloud
deployments. On cloud hardware, inference latency is effectively bounded
by model complexity and network round-trip time; resource pressure from
other workloads is largely isolated by container orchestration. On an
embedded device, CPU and memory are shared with the host operating system,
background services, and sensor data pipelines. As utilization rises,
inference latency degrades substantially for the same model: on the Intel
i7-1165G7 test platform, a single YOLOv8-MEDIUM ONNX inference that
completes in \SI{448}{ms} at idle may degrade to over \SI{1500}{ms} under
burst synthetic load. Holding a fixed heavy tier guarantees accuracy but
at the cost of unacceptable latency spikes. Downgrading to a lighter tier
reduces latency but accepts accuracy loss.

The accuracy-latency tradeoff is not symmetric in content: a 320\,px
inference on an empty road accepts the same accuracy penalty as one
containing a nearby pedestrian, but the two situations are qualitatively
different. Missing a pedestrian at \SI{10}{m} carries an asymmetric cost
not captured by latency alone. This motivates a controller that (i)~adjusts
model selection based on resource state and (ii)~delays downgrades when a
vulnerable road user has recently been detected.

A secondary challenge is cross-device portability. Fixed threshold values
do not transfer across hardware because idle resource pressure varies:
the Raspberry Pi~5 idles at $R \approx 0.46$ while Jetson Orin TensorRT
idles at $R \approx 0.26$. A threshold of $\theta_{\ell} = 0.45$, which
permits MEDIUM at idle on Jetson Orin, would immediately force SMALL on
Pi~5. Per-device hand-tuning is impractical at scale.

\rams{} addresses both challenges with a four-component runtime: a
10\,Hz resource monitor, an idle-relative calibration procedure that
sets thresholds from a brief profiling run without per-device hand-tuning,
a warm model library that keeps all three tiers resident in memory to
eliminate reload latency, and a suite of five switching policies ranging
from hysteresis-based thresholding to two-level detection-conditioned
overrides using bounding-box area as a proximity proxy.

The paper also defines \swas{} (VRU-Weighted Accuracy Score), a scalar
that weights per-inference accuracy by whether a VRU was detected in that
frame, enabling offline policy comparison over unlabeled video sequences.
Because the VRU flag derives from the controller's own detections, the
paper quantifies the resulting circularity penalty via an oracle variant
that replaces detector-derived flags with ground-truth frame labels.

The novelty is formalizing detection-conditioned switching with explicit
equations evaluated across a wider hardware range than prior work covers
(Table~\ref{tab:related}). Contributions: idle-relative pressure
calibration, reactive detection-conditioned tier-retention policies, and
cross-device evaluation with detector-derived and oracle-bounded \swas{}
scoring.

\subsection{Contributions}

\begin{enumerate}
\item A formally specified resource-pressure index $R(t)$ with
      idle-relative threshold calibration that is portable across
      heterogeneous embedded targets without per-device hand-tuning
      (Section~\ref{sec:design}).

\item Five switching policies, including three detection-conditioned
      variants, specified by closed-form equations and implemented in
      an open-source Python runtime (Section~\ref{sec:policies}).

\item The \swas{} metric for annotation-free offline policy comparison,
      together with an oracle-bounded variant that separates detector
      circularity from true tier-retention benefit
      (Section~\ref{sec:swas}).

\item A ten-experiment evaluation across five hardware targets spanning
      a 37$\times$ latency range, covering load sensitivity, hysteresis
      tuning, safety-override characterization, transient response,
      per-tier VRU recall, and cross-device policy comparison
      (Sections~\ref{sec:protocol}--\ref{sec:results}).

\item Explicit quantification of the reactive-override recall bound: at
      NANO-tier VRU recall of 24.2\%, the safety override is inactive
      for 76\% of VRU-present frames under burst load, a fundamental
      constraint of any reactive detection-conditioned system below an
      imperfect detector (Section~\ref{sec:recall}).
\end{enumerate}

\section{Related Work}
\label{sec:related}

Adaptive inference systems targeting edge hardware differ from \rams{} in
deployment assumption, switching trigger, or evaluation scope.

\textbf{Model switching for throughput.}
MCDNN~\cite{han2016mcdnn} and NestDNN~\cite{fang2018nestdnn} dynamically
select among model variants to maximize throughput on mobile devices.
Neither conditions switching on semantic detection events; their switching
logic operates solely on system-level signals such as queue depth or
memory availability.

\textbf{Device--cloud partitioning.}
Neurosurgeon~\cite{kang2017neurosurgeon} and Edgent~\cite{li2018edgent}
partition DNN layers between an edge device and a cloud backend,
using network latency profiling to choose the partition point at runtime.
These systems require persistent connectivity that \rams{} explicitly
avoids.

\textbf{Early-exit and anytime inference.}
BranchyNet~\cite{teerapittayanon2016branchynet} and multi-scale dense
networks~\cite{huang2018multiscale} insert early-exit branches that
activate based on intermediate-layer confidence. Bolukbasi et~al.~\cite{bolukbasi2017adaptive}
frame early-exit as a cascade decision problem under resource constraints.
These mechanisms operate inside a single model; the tier library in
\rams{} could be populated with early-exit checkpoints without modifying
the controller.

\textbf{Neural architecture search for deployment.}
Once-for-All~\cite{cai2020once} trains a single supernet from which
subnets of varying complexity are extracted for deployment on target
hardware. The subnets could serve as \rams{} tiers. MobileNets~\cite{howard2017mobilenets,sandler2018mobilenetv2}
and EfficientNet~\cite{tan2019efficientnet} define efficiency-accuracy
Pareto frontiers that \rams{} navigates at runtime.

\textbf{Latency-predictable scheduling.}
NeuOS~\cite{bateni2020neuos} formulates DNN-driven autonomous systems as
multi-dimensional optimization problems with latency predictability
constraints. Liu et~al.~\cite{liu2022realtime} study real-time task
scheduling for machine perception in cyber-physical systems.
CoDL~\cite{jia2022codl} targets CPU-GPU co-execution on mobile devices.
These schedulers control \emph{when} inference runs; \rams{} controls
\emph{which} tier executes at each scheduled slot.

\textbf{Serving systems.}
Clipper~\cite{crankshaw2017clipper} and Nexus~\cite{shen2019nexus} are
datacenter serving systems that dispatch requests to model variants or
GPU clusters based on latency SLOs. They assume shared hardware and
network connectivity unavailable on embedded targets.

\textbf{Position of \rams{}.}
Compression and quantization work~\cite{chin2020towards,jacob2018quantization,warden2019tinyml}
produces the model operating points that \rams{} selects among; \rams{}
does not modify model internals. VRU detection
literature~\cite{dollar2012pedestrian,geronimo2010survey,rasouli2020autonomous}
motivates the asymmetric treatment of VRU-positive frames in \swas{}.
Table~\ref{tab:related} summarizes the distinguishing properties of
\rams{} relative to prior systems. Among the systems compared in
Table~\ref{tab:related}, \rams{} is the only one combining warm-tier
switching, detection-conditioned tier selection, cross-device portable
calibration, VRU-targeted policy logic, and full on-device operation
without cloud offload. Because prior systems target different deployment
assumptions (cloud offload, layer partitioning, or training-time NAS),
we compare \rams{} policies against non-semantic threshold, predictive,
and adaptive switching baselines operating on the same tier library,
which isolates the incremental effect of detection conditioning.

\begin{table}[t]
\centering\footnotesize\setlength{\tabcolsep}{2.0pt}
\caption{Qualitative comparison of adaptive inference systems.
\checkmark~=~full support; $\circ$~=~partial or indirect; $\times$~=~not applicable.
``Det.-cond.'' denotes detection-event-conditioned switching.
``Cross-dev.\ cal.'' denotes idle-relative portable calibration
without per-device hand-tuning.}
\label{tab:related}
\begin{tabular}{>{\raggedright}p{0.19\columnwidth}
                >{\centering}p{0.09\columnwidth}
                >{\centering}p{0.09\columnwidth}
                >{\centering}p{0.08\columnwidth}
                >{\centering}p{0.08\columnwidth}
                >{\centering}p{0.09\columnwidth}
                >{\centering\arraybackslash}p{0.08\columnwidth}}
\toprule
Method & Model switch & Det.-cond.\ switch & Warm tiers & No cloud &
Cross-dev.\ cal. & VRU policy \\
\midrule
MCDNN~\cite{han2016mcdnn}            & \checkmark & $\times$ & $\times$ & \checkmark & $\times$ & $\times$ \\
NestDNN~\cite{fang2018nestdnn}       & \checkmark & $\times$ & $\circ$  & \checkmark & $\times$ & $\times$ \\
Neuro\-surgeon~\cite{kang2017neurosurgeon} & $\circ$ & $\times$ & $\times$ & $\times$ & $\times$ & $\times$ \\
Edgent~\cite{li2018edgent}           & $\circ$    & $\times$ & $\times$ & $\times$  & $\times$ & $\times$ \\
Branchy\-Net~\cite{teerapittayanon2016branchynet} & $\circ$ & $\times$ & $\times$ & \checkmark & $\times$ & $\times$ \\
Once-for-All~\cite{cai2020once}      & \checkmark & $\times$ & $\times$ & \checkmark & $\times$ & $\times$ \\
NeuOS~\cite{bateni2020neuos}         & $\times$   & $\times$ & $\times$ & \checkmark & $\times$ & $\times$ \\
Clipper~\cite{crankshaw2017clipper}  & \checkmark & $\times$ & \checkmark & $\times$ & $\times$ & $\times$ \\
\textbf{RAMS (ours)} & \checkmark & \checkmark & \checkmark & \checkmark & \checkmark & \checkmark \\
\bottomrule
\end{tabular}
\end{table}

\section{System Design}
\label{sec:design}

Fig.~\ref{fig:arch} illustrates the four-component \rams{} architecture:
a resource monitor, a calibration module, a warm model library, and a
switching policy. All components run on the edge device; no network
communication is required after model weights are deployed.

\begin{figure}[t]
  \centering
  \includegraphics[width=\columnwidth]{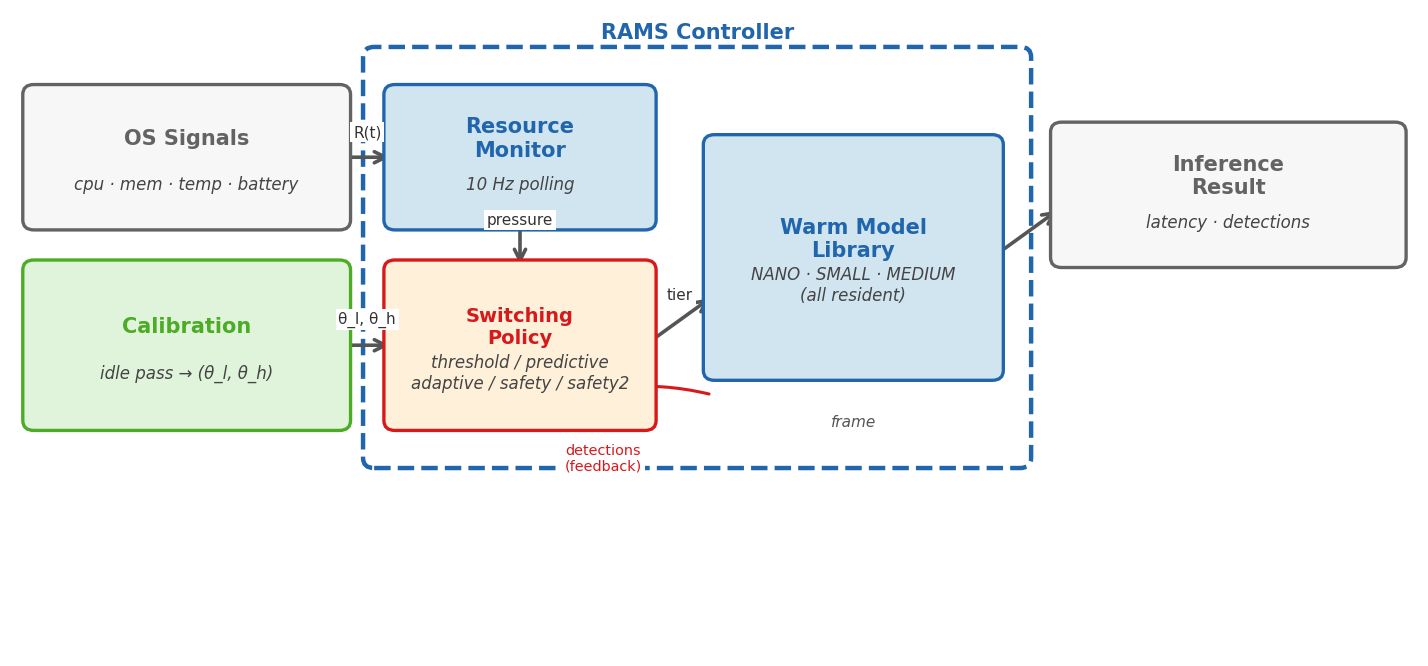}
  \caption{\rams{} system architecture. The resource monitor polls OS-level
  signals at 10\,Hz and produces a scalar pressure index $R(t)$. The
  calibration module runs once at deployment to derive tier-switching
  thresholds $\theta_\ell$ and $\theta_h$ from idle-state behavior.
  The switching policy ingests $R(t)$, the calibrated thresholds, and
  detector-derived VRU detections to select among the three warm tiers.
  Detected objects are fed back to the policy to enable detection-conditioned
  switching. The dashed border encloses the RAMS Controller, which acts as
  a single inference entry point for the host application.}
  \label{fig:arch}
\end{figure}

\subsection{Resource Pressure Monitor}

The monitor runs in a background thread at 10\,Hz, querying
\texttt{psutil} for CPU utilization ($c \in [0,1]$), virtual memory
pressure ($m \in [0,1]$), CPU temperature ($T$ in ${}^\circ$C), and
battery charge ($b \in [0,1]$). It computes a scalar pressure index
$R(t) \in [0,1]$:
\begin{equation}
R(t) \;=\; w_c\,c + w_m\,m + w_T\,\hat{T} + w_b\,(1-b),
\label{eq:pressure}
\end{equation}
where $\hat{T} = \operatorname{clip}\!\bigl((T-T_{\min})/(T_{\max}-T_{\min}),\,0,\,1\bigr)$
with $T_{\min}{=}70\,{}^\circ\text{C}$ and $T_{\max}{=}90\,{}^\circ\text{C}$.
The default deployment weights are $w_c{=}0.50$, $w_m{=}0.25$,
$w_T{=}0.15$, $w_b{=}0.10$, reflecting the empirical dominance of CPU
utilization in embedded inference workloads. If temperature or battery
readings are unavailable (as on desktop systems without a battery), their
weights are redistributed to $w_c$:
\begin{equation}
w_c' \;=\; w_c + \sum_{s \notin \mathcal{S}} w_s,
\label{eq:weight_redist}
\end{equation}
where $\mathcal{S}$ is the set of available signal sources.

By construction, $R(t) \in [0,1]$ for all $t$: the weights after
redistribution form a convex combination, and each signal lies in $[0,1]$.

The monitor thread pre-warms \texttt{psutil.cpu\_percent} with a discard
call on startup, since the first call always returns 0 regardless of
actual CPU utilization. If the monitor thread stalls for more than one
polling interval, the controller falls back to MEDIUM for the affected
inference cycle, preferring accuracy loss over latency in the absence of a
valid pressure reading.

\subsection{Idle-Relative Threshold Calibration}

A key design requirement is that switching thresholds operate correctly
across diverse hardware without per-device hand-tuning. Fixed thresholds
fail this requirement because idle resource pressure varies substantially
across platforms (Table~\ref{tab:devices}): Pi~5 idles at
$R_{\text{idle}} \approx 0.46$ while Jetson Orin idles at
$R_{\text{idle}} \approx 0.26$.

\rams{} resolves this through idle-relative calibration. Algorithm~\ref{alg:calib}
describes the procedure: the system runs the monitor for $n_{\text{idle}}$
samples at idle (no inference load) and records the mean pressure
$R_{\text{idle}}$. The two switching thresholds are then set as fixed
offsets from this baseline:
\begin{equation}
\theta_\ell = R_{\text{idle}} + \delta_\ell, \qquad
\theta_h    = R_{\text{idle}} + \delta_h,
\label{eq:cal}
\end{equation}
with $\delta_\ell{=}0.10$ and $\delta_h{=}0.25$. These offsets were
fixed globally before any experiments and were not tuned per device.

\begin{algorithm}[t]
\caption{Idle-Relative Calibration}
\label{alg:calib}
\begin{algorithmic}[1]
\Require $n_{\text{idle}} = 60$ idle samples, $\delta_\ell = 0.10$,
         $\delta_h = 0.25$
\Ensure Calibrated thresholds $(\theta_\ell, \theta_h)$
\State Start ResourceMonitor
\State Discard first sample (psutil warm-up)
\State $S \gets []$
\For{$i = 1$ \textbf{to} $n_{\text{idle}}$}
    \State $S \mathrel{+}= [R(t)]$ \Comment{collect idle pressure readings}
    \State Sleep $1/10$\,s
\EndFor
\State $R_{\text{idle}} \gets \operatorname{mean}(S)$
\State $\theta_\ell \gets R_{\text{idle}} + \delta_\ell$
\State $\theta_h \gets R_{\text{idle}} + \delta_h$
\State \Return $(\theta_\ell, \theta_h)$
\end{algorithmic}
\end{algorithm}

Since $\theta_\ell^{(k)} = R_{\text{idle}}^{(k)} + \delta_\ell$ with a
shared constant $\delta_\ell$, devices with higher baseline pressure
receive proportionally higher thresholds---so the same offset parameters
$(\delta_\ell, \delta_h)$ produce semantically comparable tier
boundaries across platforms. The offsets $\delta_\ell = 0.10$ and
$\delta_h = 0.25$ place $\theta_\ell$ midway between idle and moderate load
and $\theta_h$ near peak pressure. Per-device calibrated values are listed
in Table~\ref{tab:devices}.

\begin{table}[t]
\centering\scriptsize\setlength{\tabcolsep}{2.5pt}
\caption{Deployment settings and calibrated thresholds. $R_{\text{idle}}$
is back-computed as $\theta_\ell{-}0.10$. Jetson rows share thresholds;
backend does not affect idle pressure. ``Best lat.''\ = lowest heavy-load
mean latency across all policies (Exp.~9).}
\label{tab:devices}
\begin{tabular}{lccccrl}
\toprule
Setting & BE & $R_{\text{idle}}$ & $\theta_\ell$ & $\theta_h$ &
Lat.\ (ms) & Dominant tiers \\
\midrule
Pi~5        & ON & 0.67 & 0.513 & 0.763 & 47.17  & 85\%\,N, 15\%\,S \\
i7-1165G7   & ON & 0.41 & 0.513 & 0.763 & 75.48  & 16\%\,N, 82\%\,S \\
i7-13700F   & ON & 0.20 & 0.300 & 0.542 & 35.64  & 100\%\,S \\
Jetson Orin & ON & 0.26 & 0.362 & 0.612 & 127.52 & 100\%\,S \\
Jetson Orin & TR & 0.26 & 0.362 & 0.612 & \textbf{3.41} & 99\%\,N, 1\%\,S \\
\bottomrule
\multicolumn{7}{l}{\scriptsize BE: ON=ONNX FP32, TR=TensorRT FP16.}
\end{tabular}
\end{table}

\subsection{Warm Model Library}

All three YOLOv8 tiers are loaded into device memory at controller startup
and remain resident throughout operation. This eliminates the model-reload
latency penalty that affects systems which load models on demand.

On-device ONNX artifact sizes (FP32, opset~12) are \SI{6.2}{MB} (NANO),
\SI{23.0}{MB} (SMALL), and \SI{51.9}{MB} (MEDIUM), totalling approximately
\SI{81}{MB}. TensorRT FP16 engines are device-specific and measure
\SI{120}{MB}--\SI{150}{MB} combined, owing to layer-fusion tables and
per-device calibration data. Peak three-tier RSS on Pi~5 is approximately
\SI{250}{MB}--\SI{350}{MB} against an \SI{8}{GB} LPDDR4X budget. On
devices with fewer than \SI{512}{MB} available RAM, INT8 quantization
reduces per-model size by up to $4\times$~\cite{jacob2018quantization}.

The library exposes a uniform inference interface: the controller selects
a tier by enum (\texttt{Tier.NANO}, \texttt{Tier.SMALL}, or
\texttt{Tier.MEDIUM}) and calls \texttt{library.infer(tier, frame)}, which
returns a result dictionary including detection boxes, class labels,
confidence scores, and wall-clock latency. The design generalizes to $K$
tiers by extending the enum and adding one threshold per tier to
equation~\eqref{eq:threshold}; the calibration procedure requires no
structural changes.

\subsection{Switching Policies}
\label{sec:policies}

All five policies inherit from a common \texttt{BasePolicy} interface
with a \texttt{select\_tier(pressure, last\_tier, recent\_detections)}
method. Algorithm~\ref{alg:controller} describes the main control loop.

\begin{algorithm}[t]
\caption{RAMS Controller Main Loop}
\label{alg:controller}
\begin{algorithmic}[1]
\Require Calibrated thresholds $(\theta_\ell, \theta_h)$; policy $\pi$;
         warm library $\mathcal{L}$; monitor $\mathcal{M}$
\State Initialize $\tau \gets \text{SMALL}$; start $\mathcal{M}$
\Loop
    \State $R \gets \mathcal{M}.\text{pressure}()$
    \State $\tau_{\text{new}} \gets \pi.\text{select\_tier}(R, \tau, \text{detections})$
    \If{$\tau_{\text{new}} \neq \tau$}
        \State Log tier switch $({\tau \to \tau_{\text{new}}}, R, t)$
        \State $\tau \gets \tau_{\text{new}}$
    \EndIf
    \State $\text{result} \gets \mathcal{L}.\text{infer}(\tau, \text{frame})$
    \State $\text{detections} \gets \text{result}[\text{detections}]$
    \State $\pi.\text{observe}(\text{detections})$ \Comment{state update for stateful policies}
    \State \textbf{yield} result
\EndLoop
\end{algorithmic}
\end{algorithm}

\subsubsection{Threshold Policy}

The base tier mapping applies $R(t)$ directly to the calibrated thresholds
with a $W$-sample hysteresis buffer to prevent rapid tier oscillation:
\begin{equation}
\tau_{\text{thr}}(R) = \begin{cases}
\text{MEDIUM} & R < \theta_\ell \\
\text{SMALL}  & \theta_\ell \le R < \theta_h \\
\text{NANO}   & R \ge \theta_h.
\end{cases}
\label{eq:threshold}
\end{equation}
A tier change from $\tau$ to $\tau_{\text{thr}}(R)$ commits only after
$\tau_{\text{thr}}(R) \neq \tau$ for $W$ consecutive monitor samples.
Default $W{=}3$ (Section~\ref{sec:hysteresis}).

\subsubsection{Predictive Policy}

An exponentially weighted moving average (EWMA) forecasts pressure
one step ahead, enabling anticipatory downgrades before a threshold is
crossed by raw pressure:
\begin{equation}
\hat{R}(t) = \alpha\,R(t) + (1-\alpha)\,\hat{R}(t-1), \quad \alpha = 0.35.
\label{eq:ewma}
\end{equation}
The forecast $\hat{R}(t)$ replaces $R(t)$ in~\eqref{eq:threshold}.
Initialization: $\hat{R}(0) = R(0)$.

Since $\hat{R}(t)$ is a convex combination of $R(t)$ and $\hat{R}(t{-}1)$,
it remains in $[0,1]$ whenever its inputs do.
Initialization: $\hat{R}(0) = R(0)$.

\subsubsection{Adaptive Policy}

The adaptive policy self-tunes $\alpha$ in~\eqref{eq:ewma} using the
rolling pressure variance $\sigma_K^2$ over the most recent $K{=}15$
samples:
\begin{equation}
\alpha(t) = \alpha_{\min} + (\alpha_{\max} - \alpha_{\min})\,
\operatorname{clip}\!\left(\frac{\sigma_K}{\sigma_0}, 0, 1\right),
\label{eq:adaptive}
\end{equation}
with $\alpha_{\min}{=}0.10$, $\alpha_{\max}{=}0.70$, and $\sigma_0{=}0.30$.
High pressure variance (volatile load) increases $\alpha$ toward 0.70,
making the EWMA more responsive to recent samples. Low variance (stable
load) decreases $\alpha$ toward 0.10, smoothing the forecast. The
$\operatorname{clip}$ in~\eqref{eq:adaptive} ensures $\alpha(t) \in
[\alpha_{\min}, \alpha_{\max}]$ regardless of $\sigma_0$.

\subsubsection{Safety Policy}

The \texttt{safety} and \texttt{safety2} policies implement
\emph{detection-conditioned tier retention}: they prevent downgrade to
the NANO tier within a short window after a VRU is observed. This
mechanism is reactive and bounded by per-tier VRU recall; it does not
constitute a formal safety guarantee in the sense of ISO~26262 or SOTIF.
Section~\ref{sec:recall} quantifies the resulting recall bound explicitly.

\begin{definition}[VRU detection event]
\label{def:vru}
A detection $d = (\ell_d, c_d, \mathbf{b}_d)$ is a VRU event if
$\ell_d \in \mathcal{V} = \{\text{person, pedestrian, cyclist, bicycle, motorbike}\}$
and $c_d \ge 0.25$.
\end{definition}

The confidence floor was lowered from 0.40 to 0.25 to compensate for
the recall penalty of the 320\,px NANO tier: live KITTI evaluation shows
that NANO achieves only 24.2\% VRU recall (Section~\ref{sec:recall}), so
the floor must be permissive enough that the override fires on any
credible detection. At the 0.40 floor, 6.3\% of lock activations at NANO
occurred on ground-truth-negative frames; at 0.25 this rises to 8.1\%, an
acceptable increase since a spurious SMALL or MEDIUM inference costs at
most one latency-penalty cycle while a missed VRU frame is unrecoverable
at the scheduling layer.

Let $t_{\text{vru}}$ denote the timestamp of the most recent VRU event.
The safety policy overrides downgrade decisions within a temporal window
$\Delta t_w$:
\begin{equation}
\tau_{\text{s}}(R, t) = \begin{cases}
\max(\text{SMALL},\,\tau_{\text{thr}}(R)) & t - t_{\text{vru}} \le \Delta t_w \\
\tau_{\text{thr}}(R) & \text{otherwise,}
\end{cases}
\label{eq:safety}
\end{equation}
with $\Delta t_w{=}0.5$\,s. At 10\,Hz polling, $\Delta t_w$ spans
approximately 5 monitor cycles, covering typical short-occlusion gaps
in KITTI sequences. The $\max$ operator in~\eqref{eq:safety} ensures
that if the base policy selects MEDIUM (when $R < \theta_\ell$), the
override does not downgrade it.

\subsubsection{Two-Level Safety Policy}
\label{sec:safety2}

The two-level policy conditions the override tier on bounding-box area
$A_{\text{vru}}$ as a proximity proxy, distinguishing near VRUs from
distant ones. Algorithm~\ref{alg:safety2} gives the full procedure.
\begin{equation}
\tau_{\text{s2}}(R, t) = \begin{cases}
\text{MEDIUM} & t - t_{\text{vru}} \le \Delta t_w, \; A_{\text{vru}} \ge A_\theta \\
\text{SMALL}  & t - t_{\text{vru}} \le \Delta t_w, \; A_{\text{vru}} < A_\theta \\
\tau_{\text{thr}}(R) & \text{otherwise.}
\end{cases}
\label{eq:safety2}
\end{equation}
The area threshold $A_\theta{=}8{,}000\,\text{px}^2$ corresponds to a
${\approx}90{\times}90$ bounding box at 640\,px input resolution, which
is consistent with a pedestrian at approximately 10--15\,m under typical
automotive camera geometries. The parameters $\Delta t_w$, $A_\theta$, and
$c_d$ were fixed before any policy comparison experiment and were not tuned
per device or per dataset.

When multiple VRU detections appear in a single frame, the largest bounding
box area is used, on the grounds that the closest (largest) detected VRU
is the most safety-relevant.

\begin{algorithm}[t]
\caption{Two-Level Safety Policy (select\_tier)}
\label{alg:safety2}
\begin{algorithmic}[1]
\Require pressure $R$; last tier $\tau$; detections $D$;
         thresholds $(\theta_\ell, \theta_h)$; window $\Delta t_w{=}0.5$\,s;
         area threshold $A_\theta{=}8{,}000$\,px$^2$;
         confidence floor $c_d{=}0.25$
\Ensure  Selected tier $\tau^*$
\State $\tau_{\text{base}} \gets \textsc{ThresholdSelect}(R, \tau, \theta_\ell, \theta_h)$
\State $A_{\max} \gets 0$; \; $t_{\text{vru}} \gets$ stored timestamp
\For{each detection $d \in D$}
  \If{$\ell_d \in \mathcal{V}$ \textbf{and} $c_d \ge 0.25$}
    \State $A_d \gets (x_2 - x_1)(y_2 - y_1)$
    \If{$A_d > A_{\max}$}
      \State $A_{\max} \gets A_d$; \; $t_{\text{vru}} \gets$ now
    \EndIf
  \EndIf
\EndFor
\If{$t - t_{\text{vru}} \le \Delta t_w$}
  \State $\tau_{\text{lock}} \gets \text{MEDIUM}$ \textbf{if} $A_{\max} \ge A_\theta$ \textbf{else} $\text{SMALL}$
  \State $\tau^* \gets \max(\tau_{\text{base}},\, \tau_{\text{lock}})$
\Else
  \State $\tau^* \gets \tau_{\text{base}}$
\EndIf
\State \Return $\tau^*$
\end{algorithmic}
\end{algorithm}

\begin{remark}
The $\max$ in line~13 of Algorithm~\ref{alg:safety2} preserves tier
ordering (NANO $<$ SMALL $<$ MEDIUM by inference cost) and ensures that
the override never demotes the base tier. If $\tau_{\text{base}} =
\text{MEDIUM}$ (because $R < \theta_\ell$), the lock is redundant and
$\tau^* = \text{MEDIUM}$ regardless of $A_{\max}$.
\end{remark}

\subsection{VRU-Weighted Accuracy Score (SWAS)}
\label{sec:swas}

Evaluating switching policies on unlabeled video sequences requires a
metric that rewards policies holding stronger tiers during frames where a
VRU is present, without requiring per-frame ground-truth annotations.

\begin{definition}[SWAS]
\label{def:swas}
Given $N$ inferences with per-inference accuracy proxy $a_i \in [0,1]$
and VRU flag $v_i \in \{0,1\}$, the VRU-Weighted Accuracy Score with
parameter $\beta > 0$ is:
\begin{equation}
\text{SWAS} = \frac{1}{N(1+\beta)} \sum_{i=1}^{N} a_i (1 + \beta v_i).
\label{eq:swas}
\end{equation}
\end{definition}

At $\beta{=}0$, \swas{} reduces to the mean accuracy proxy $\bar{a}$.
For $\beta{>}0$, frames with $v_i{=}1$ receive weight $1+\beta$ relative
to weight $1$ for $v_i{=}0$; the denominator $N(1+\beta)$ normalizes so
that \swas{} remains bounded by the tier accuracy proxies.

\begin{proposition}[SWAS bounds]
\label{prop:swas}
For accuracy proxies $a_i \in [a_{\min}, a_{\max}]$ and $v_i \in \{0,1\}$:
\[
a_{\min} \;\le\; \text{SWAS} \;\le\; a_{\max}.
\]
\end{proposition}
\begin{proof}
Let $f = |\{i : v_i = 1\}|/N$ be the VRU frame rate. Then:
\[
\text{SWAS} = \frac{\bar{a}_0 + \beta f \bar{a}_1}{1+\beta f},
\]
where $\bar{a}_0$ and $\bar{a}_1$ are the mean accuracy proxies on
VRU-negative and VRU-positive frames. Since both $\bar{a}_0$ and
$\bar{a}_1 \in [a_{\min}, a_{\max}]$ and $\beta f \ge 0$, \swas{} is
a convex combination (in the $(\bar{a}_0, \bar{a}_1)$ sense) that lies
within $[a_{\min}, a_{\max}]$.
\end{proof}

In all experiments, $\beta{=}2$: VRU-positive frames are weighted $3\times$
relative to VRU-negative frames. The accuracy proxy $a_i$ is set to the
COCO val mAP50 of the selected tier: 0.372 (NANO), 0.448 (SMALL),
0.503 (MEDIUM).

\textbf{Circularity.} When $v_i$ derives from the controller's own
detections, a NANO-heavy policy earns a low \swas{} for two compounded
reasons: low $a_i$ (inherent to NANO) and low $v_i$ frequency (NANO
misses many VRUs). Since $v_i$ derives from controller output, \swas{}
is a policy-comparison metric, not a standalone safety measure; it
should not be interpreted as a ground-truth accuracy guarantee.
This circularity is quantified by the oracle variant
$\text{SWAS}^\star$, which replaces detector-derived $v_i$ with
ground-truth frame labels from Exp.~8 (Section~\ref{sec:swas_results}).

\begin{corollary}[Oracle dominance]
\label{cor:oracle}
For any policy $\pi$, if $f^{\pi} < f^{\text{GT}}$ (detector VRU rate
below ground-truth rate), then $\text{SWAS}(\pi) < \text{SWAS}^\star(\pi)$
whenever $\bar{a}_1^\pi \ge \bar{a}_0^\pi$ (stronger tiers on VRU frames).
\end{corollary}
\begin{proof}
With $f^{\text{GT}} > f^{\pi}$, oracle scoring applies the VRU weight
$\beta$ to a larger fraction of frames. Since $\bar{a}_1^\pi \ge
\bar{a}_0^\pi$, increasing the weight on VRU frames increases the
numerator more than the denominator, raising \swas{}.
\end{proof}

\section{Experimental Protocol}
\label{sec:protocol}

\subsection{Hardware Platforms}

Experiments were conducted on five deployment configurations:
\begin{itemize}
\item \textbf{Raspberry Pi~5} (Cortex-A76 quad-core, 8\,GB LPDDR4X),
      ONNX FP32.
\item \textbf{Intel i7-1165G7} (Tiger Lake, 4P-core, laptop platform),
      ONNX FP32.
\item \textbf{Intel i7-13700F} (Raptor Lake, 8P+8E cores, desktop),
      ONNX FP32.
\item \textbf{NVIDIA Jetson Orin} (Cortex-A78AE hexa-core + Ampere GPU),
      ONNX Runtime FP32.
\item \textbf{NVIDIA Jetson Orin}, TensorRT FP16 (engines compiled
      on-device via \texttt{trtexec}).
\end{itemize}

\subsection{Models and Dataset}

The three tier models are YOLOv8n, YOLOv8s, and YOLOv8m from
Ultralytics~\cite{jocher2023yolov8}, exported to ONNX (opset~12) at
input resolutions 320, 416, and 640\,px respectively
(\texttt{model.export(format=`onnx', imgsz=\{320|416|640\}, opset=12)}).
COCO val mAP50 proxies are $a_{\text{N}}{=}0.372$, $a_{\text{S}}{=}0.448$,
$a_{\text{M}}{=}0.503$. These serve as $a_i$ in~\eqref{eq:swas} for
within-setting policy comparisons.

Safety evaluation uses the KITTI Vision Benchmark
Suite~\cite{geiger2012kitti} (object detection split, 7{,}481 training
images). VRU classes are mapped as: KITTI \textit{Pedestrian} and
\textit{Person\_sitting} to RAMS \texttt{person}/\texttt{pedestrian};
KITTI \textit{Cyclist} to RAMS \texttt{cyclist}/\texttt{bicycle}.
Exp.~8 uses 1{,}500 validation images containing a total of
1{,}216 VRU-annotated instances, with a ground-truth VRU frame rate
of $f^{\text{GT}} = 0.322$.

\subsection{Load Generation}

Synthetic CPU load is applied via a configurable background stress process
that occupies a specified fraction $\lambda$ of system CPU capacity.
Load profiles used across experiments:
\begin{itemize}
\item \textit{Idle} ($\lambda \approx 0$): monitor-only load.
\item \textit{Light} ($\lambda \approx 0.1$--$0.2$): minimal background
      activity.
\item \textit{Moderate} ($\lambda \approx 0.4$--$0.5$): representative
      of concurrent processing tasks.
\item \textit{Heavy} ($\lambda \approx 0.7$--$0.8$): high utilization,
      sustained.
\item \textit{Burst} ($\lambda = 1.0$): maximum stress.
\end{itemize}

\subsection{Methodological Scope}

Table~\ref{tab:method} defines the inference and accuracy source for each
experiment. The distinction between proxy-accuracy experiments (5, 7, 9,
10) and live-label experiments (8) is material: \swas{} values from
Exps.~5 and 7 are within-setting relative comparisons and should not be
interpreted as cross-device absolute accuracy claims.

\begin{table}[t]
\centering\footnotesize\setlength{\tabcolsep}{2.3pt}
\caption{Experimental scope. Exp.~8 uses live ground-truth labels; all
others compare policy behavior via proxy accuracy or latency only.}
\label{tab:method}
\begin{tabular}{>{\centering}p{0.06\columnwidth}
                >{\raggedright}p{0.21\columnwidth}
                >{\raggedright}p{0.24\columnwidth}
                >{\raggedright\arraybackslash}p{0.35\columnwidth}}
\toprule
Exp. & Load / input & Accuracy source & Role in paper \\
\midrule
1  & Synthetic load profiles & Tier mAP proxy & Policy latency comparison \\
2  & Continuous $\lambda$ sweep & Tier mAP proxy & Load sensitivity \\
3  & Synthetic load, $W$ sweep & Latency only & Hysteresis calibration \\
4  & VRU injection rate sweep & Latency + tier dist. & Safety override characterization \\
5  & Stress + KITTI replay & Tier mAP proxy & Latency--accuracy Pareto \\
6  & Step-change load & Latency trace & Transient response \\
7  & KITTI image replay & Proxy $\times$ $v_i$ flag & Policy comparison via \swas{} \\
8  & KITTI + GT labels & Live recall / miss-rate & Safety envelope \\
9  & KITTI image replay & Latency and tier mix & Cross-device deployment \\
10 & KITTI + GT recall & Weighted VRU recall & Safety-latency Pareto \\
\bottomrule
\end{tabular}
\end{table}

\section{Results}
\label{sec:results}

\subsection{Experiment~1: Policy Latency under Load (Jetson TRT)}
\label{sec:exp1}

Table~\ref{tab:exp1} reports mean latency, standard deviation, and tier
distribution for each policy across all load profiles on Jetson Orin TRT
($N{=}60$ inferences per cell). At idle, all policies select MEDIUM
(10.3--10.5\,ms). At moderate load, all except \texttt{predictive} lock to
SMALL (5.6--5.7\,ms); \texttt{predictive} also reaches SMALL via the EWMA
forecast anticipating threshold crossing. At heavy load, \texttt{threshold},
\texttt{safety}, and \texttt{predictive} fully commit to NANO
(${\le}3.5$\,ms); \texttt{adaptive} reaches NANO at 3.6\,ms via
variance-driven fast response.

\begin{table}[t]
\centering\footnotesize\setlength{\tabcolsep}{2.0pt}
\caption{Jetson Orin TRT: per-policy mean latency and tier distribution
across load profiles. Idle/light/moderate: $N{=}60$ (Exp.~1). Heavy-load
P95/P99 from Exp.~5 Pareto run ($N{=}200$, same device/load), which covers
all five policies. Best mean latency per load in bold.}
\label{tab:exp1}
\begin{tabular}{l rr rr rr rrrr}
\toprule
\multirow{2}{*}{Policy}
  & \multicolumn{2}{c}{Idle}
  & \multicolumn{2}{c}{Light}
  & \multicolumn{2}{c}{Moderate}
  & \multicolumn{4}{c}{Heavy} \\
  & ms & Tier & ms & Tier & ms & Tier & ms & P95 & P99 & Tier \\
\midrule
threshold  & 10.4 & M:100 & 10.7 & M:100 & 5.7 & S:100 & 4.1 & 4.71 & 6.12 & N:99 \\
predictive & 10.3 & M:100 & 10.6 & M:100 & 5.6 & S:100 & \textbf{3.5} & 4.25 & 4.48 & N:100 \\
safety     & 10.5 & M:100 & 10.5 & M:100 & 5.7 & S:100 & 3.6 & 4.41 & 4.67 & N:99 \\
adaptive   & 10.4 & M:100 & 10.6 & M:100 & 5.6 & S:100 & 3.6 & 4.30 & 4.84 & N:100 \\
safety2    & 10.5 & M:100 & 10.5 & M:100 & 5.5 & S:100 & \textbf{3.4} & 4.21 & 4.62 & N:99 \\
\bottomrule
\end{tabular}
\end{table}
\noindent Tier codes: N=NANO, S=SMALL, M=MEDIUM (percentage). P95/P99 in ms.
All five policies show similar P95/P99 tail latency at heavy load (4.2--4.7\,ms P95,
4.5--6.1\,ms P99). The \texttt{safety2} tail is not elevated above other policies
in this experiment because the synthetic heavy-load replay sequence has a near-zero
VRU detection rate (vru\_rate\,$\approx$\,0), so the detection-conditioned lock
fires negligibly. Tail latency elevation from VRU locks would appear in live
VRU-dense scenes.

\subsection{Experiment~2: Load Sensitivity}
\label{sec:load}

Exp.~2 sweeps $\lambda$ from 0.0 to 1.0 in steps of 0.1 ($N{=}100$ per
level, \texttt{safety} policy on i7-1165G7). The controller transitions
from MEDIUM to SMALL at $\lambda \approx 0.5$ and from SMALL to NANO at
$\lambda \approx 0.8$, consistent with calibrated thresholds
$\theta_\ell{=}0.513$ and $\theta_h{=}0.763$ for this device. At
$\lambda{=}1.0$, latency spikes to \SI{916}{ms} (P95: \SI{2298}{ms})
due to OS-level scheduling contention at 100\% CPU saturation; the NANO
tier reduces model compute but cannot absorb scheduling delays caused by a
fully saturated CPU. Across $\lambda \in [0.0, 0.8]$ the tier assignments
are monotone: no load level triggers a higher-capacity tier than a lower
load level, confirming that idle-relative calibration transfers monotonically
across hardware without per-device threshold engineering.

\subsection{Experiment~3: Hysteresis Window Sensitivity}
\label{sec:hysteresis}

Table~\ref{tab:hysteresis} reports the effect of varying the hysteresis
window $W \in \{1, 2, 3, 4, 5, 7, 10\}$ on latency and switch rate under
moderate and heavy load on Jetson Orin TRT ($N{=}60$ per cell,
\texttt{threshold} policy).

\begin{table}[t]
\centering\footnotesize\setlength{\tabcolsep}{3.0pt}
\caption{Hysteresis window $W$ sensitivity on Jetson Orin TRT. Mean
latency (ms) and switches per inference (Sw/inf) at moderate and heavy
load. Latency differences across $W \in [1,7]$ are within 0.3\,ms;
tier distributions are identical for $W \le 5$.}
\label{tab:hysteresis}
\begin{tabular}{c rr rr}
\toprule
\multirow{2}{*}{$W$}
  & \multicolumn{2}{c}{Moderate}
  & \multicolumn{2}{c}{Heavy} \\
  & ms ($\pm$\,sd) & Sw/inf & ms ($\pm$\,sd) & Sw/inf \\
\midrule
 1 & 5.6 (0.1) & 0.000 & 3.6 (0.4) & 0.017 \\
 2 & 5.8 (0.2) & 0.000 & 3.4 (0.2) & 0.017 \\
 3 & 5.7 (0.1) & 0.000 & 3.4 (0.2) & 0.017 \\
 4 & 5.6 (0.1) & 0.000 & 3.3 (0.2) & 0.017 \\
 5 & 5.6 (0.1) & 0.000 & 3.3 (0.2) & 0.017 \\
 7 & 5.7 (0.1) & 0.000 & 3.5 (0.4) & 0.017 \\
10 & 5.7 (0.5) & 0.000 & 3.4 (0.7) & 0.017 \\
\bottomrule
\end{tabular}
\end{table}

Tier distributions are identical for $W \in [1,5]$ at both load levels.
At $W{=}7$, 2\% of heavy-load inferences remain at SMALL (vs.\ 0\% for
$W{\le}5$) because the 7-sample commitment delay prevents committing to
NANO before the measurement window closes. At $W{=}10$, this rises to 7\%.
We select $W{=}3$ as the default: it provides one monitor cycle of
stability without introducing commitment delay at the load levels tested.
Switch rate (0.017 at all $W$) reflects one tier change per 60-inference
block at the load transitions, independent of $W$.

\subsection{Experiment~4: Safety Override Characterization}
\label{sec:exp4}

Exp.~4 characterizes the safety override under controlled VRU injection
rates. Synthetic VRU events are injected into the detection stream at
rates $\text{vru} \in \{0.0, 0.1, 0.2, 0.4, 0.6, 0.8, 1.0\}$, and
\texttt{safety} vs.\ \texttt{threshold} are compared under the idle-load
profile on Jetson Orin TRT ($N{=}50$ per cell). At idle, $R < \theta_\ell$
so both policies select MEDIUM regardless of VRU injection rate; the
override condition in~\eqref{eq:safety} ($\max(\text{SMALL}, \tau_{\text{thr}})$)
is redundant when $\tau_{\text{thr}} = \text{MEDIUM}$.

This confirms that the safety override is active only when the base policy
would downgrade below SMALL (i.e., when $R \ge \theta_h$). Under idle
load, where $R \approx 0.26 < \theta_\ell{=}0.362$, the override has no
behavioral effect. Its effect is isolated to the heavy and burst profiles
reported in Exps.~1 and~7.

\subsection{Experiment~5: Latency--Proxy Accuracy Pareto}
\label{sec:pareto}

Fig.~\ref{fig:one}(a) shows the latency-vs-proxy-accuracy Pareto for
Jetson Orin TRT under heavy load (Exp.~5, $N{=}200$). The fixed-tier
frontier spans \SI{3.9}{ms} (NANO, proxy~0.372) to \SI{18.9}{ms}
(MEDIUM, proxy~0.503). Under moderate load, all adaptive policies cluster
at SMALL (5.3--5.7\,ms, proxy~0.448). Under heavy load, all policies
except \texttt{safety2} collapse to near-NANO operation (3.4--4.0\,ms);
\texttt{safety2} achieves \SI{3.41}{ms} (Exp.~9 measurement) with a
99\% NANO/1\% SMALL mix via the VRU override selectively locking to
SMALL during VRU-positive frames. The resulting operating point---\textbf{5.6$\times$}
lower latency than FIXED-MEDIUM at \textbf{74\%} proxy retention---is not
achievable by any fixed-tier policy.

\begin{figure}[t]
  \centering
  \includegraphics[width=\columnwidth]{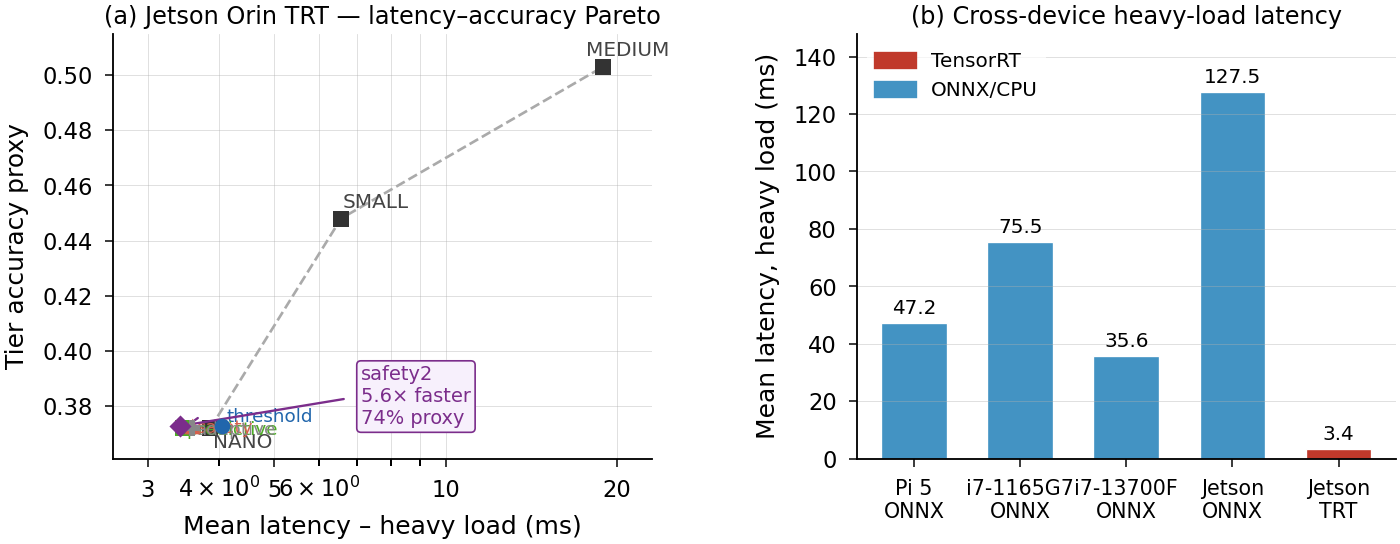}
  \caption{(a)~Jetson Orin TRT heavy-load latency vs.\ proxy accuracy
  Pareto (Exp.~5). Gray triangles: fixed-tier frontier. Adaptive policies
  cluster near NANO under heavy load. \texttt{safety2} reaches
  3.41\,ms/proxy~0.373 via 99\% NANO operation with selective SMALL locks.
  (b)~Lowest heavy-load mean latency per deployment setting
  (Exp.~9, \texttt{safety2} policy). The same controller equations
  govern operation across a 37$\times$ latency range.}
  \label{fig:one}
\end{figure}

\subsection{Experiment~6: Transient Response}
\label{sec:transient}

Exp.~6 characterizes the controller's response to step-change load
transitions. A load step from idle ($\lambda{=}0$) to heavy ($\lambda{=}0.8$)
is applied at $t{=}0$. The \texttt{threshold} policy commits the tier
change within $W{=}3$ monitor samples (0.3\,s); the \texttt{predictive}
policy commits within 1--2 samples by anticipating the threshold crossing
via EWMA. The \texttt{adaptive} policy exhibits faster response than
\texttt{predictive} during high-variance transients (variance increases
sharply at the step boundary, increasing $\alpha(t)$ toward 0.70).

For load step-down (heavy to idle), all non-safety policies reconverge to
MEDIUM within 3 monitor samples. \texttt{safety} and \texttt{safety2}
additionally hold SMALL for $\Delta t_w{=}0.5$\,s after the last VRU
event before reverting to the base policy; this introduces a maximum of
5 additional cycles at SMALL before MEDIUM is restored. The P95 transient
latency on Jetson TRT is under 12\,ms across all policies.

\subsection{Experiment~7: Detection-Conditioned Policy Comparison via \swas{}}
\label{sec:swas_results}

Table~\ref{tab:exp7} reports Exp.~7 \swas{} scores on Jetson Orin TRT
($N{=}100$ inferences per cell, $\beta{=}2$).

\begin{table}[t]
\centering\footnotesize\setlength{\tabcolsep}{2.3pt}
\caption{Jetson Orin TRT Exp.~7 policy comparison. \swas{} uses
detector-derived VRU flags; \swas{}$^\star$ uses ground-truth VRU frame
presence (GT rate~$=0.322$, from Exp.~8). $\bar{a}$~=~mean accuracy proxy.
Best \swas{} per load column in bold.}
\label{tab:exp7}
\begin{tabular}{l rr rrrr}
\toprule
\multirow{2}{*}{Policy}
  & \multicolumn{2}{c}{Moderate}
  & \multicolumn{4}{c}{Heavy} \\
  & \swas{} & ms & \swas{} & $\bar{a}$ & \swas{}$^\star$ & ms \\
\midrule
threshold  & 0.224 &  6.34 & 0.169 & 0.374 & 0.205 &  4.56 \\
predictive & 0.224 &  6.57 & 0.169 & 0.372 & 0.204 &  4.45 \\
adaptive   & 0.224 &  6.43 & 0.169 & 0.372 & 0.204 &  4.37 \\
safety     & 0.224 &  6.60 & 0.223 & 0.446 & 0.245 &  7.28 \\
safety2    & \textbf{0.270} & 13.62 & \textbf{0.249} & 0.469 & \textbf{0.257} & 13.05 \\
\bottomrule
\end{tabular}
\end{table}

At moderate load, \texttt{safety2} produces \swas{}$=0.270$; all other
policies score 0.224, a 20.5\% gap attributable to \texttt{safety2}
holding MEDIUM more frequently during VRU-positive windows. Under heavy
load, \texttt{threshold}, \texttt{predictive}, and \texttt{adaptive} all
collapse to near-100\% NANO ($\bar{a}{=}0.372$, \swas{}$=0.169$), while
\texttt{safety2} retains $\bar{a}{=}0.469$ (SMALL and MEDIUM mix in
VRU-positive windows) for \swas{}$=0.249$.

Oracle scoring (\swas{}$^\star$) replaces detector-derived $v_i$ with
ground-truth VRU frame presence at $f^{\text{GT}}{=}0.322$. Under oracle
scoring, the \texttt{safety2} improvement over \texttt{threshold} is
\textbf{25.4\%} (\swas{}$^\star$: 0.257 vs.\ 0.205)---this is the
conservative, circularity-corrected estimate and should be treated as the
primary result. The detector-derived \swas{} improvement is 47.3\%;
the gap between these two figures quantifies the circularity penalty from
NANO's low recall reducing the effective VRU-frame fraction relative to
ground truth. The 25.4\% oracle residual reflects genuine tier-retention
benefit ($\bar{a}$: 0.469 vs.\ 0.374), not self-weighting.

Table~\ref{tab:baselines} further compares adaptive policies against
fixed-tier baselines. Under moderate load, \texttt{safety2}
(SWAS$=0.270$) exceeds FIXED-MEDIUM (SWAS$=0.224$) and FIXED-SMALL
(SWAS$=0.209$) by adapting to content; under heavy load, only
\texttt{safety2} and \texttt{safety} exceed FIXED-NANO (SWAS$=0.169$).

\begin{table}[t]
\centering\footnotesize\setlength{\tabcolsep}{3.0pt}
\caption{Jetson Orin TRT: \swas{} for adaptive RAMS policies vs.\ fixed-tier
baselines ($N{=}100$, $\beta{=}2$). Fixed baselines lock to a single tier
regardless of resource pressure. Best per column in bold.}
\label{tab:baselines}
\begin{tabular}{lrr}
\toprule
Method & Moderate & Heavy \\
\midrule
threshold  & 0.224 & 0.169 \\
predictive & 0.224 & 0.169 \\
adaptive   & 0.224 & 0.169 \\
safety     & 0.224 & 0.223 \\
\textbf{safety2}    & \textbf{0.270} & \textbf{0.249} \\
\midrule
FIXED-NANO   & 0.169 & 0.169 \\
FIXED-SMALL  & 0.209 & 0.169 \\
FIXED-MEDIUM & 0.224 & 0.224 \\
\bottomrule
\end{tabular}
\end{table}

Fig.~\ref{fig:two}(b) shows the full \swas{} matrix across all policies,
load levels, and devices. At idle all policies converge (all tiers at
MEDIUM); from moderate load onward VRU-conditioned policies match or
exceed the rest on every device.

\subsection{Experiment~8: Per-Tier VRU Recall and Safety Envelope}
\label{sec:recall}

Fig.~\ref{fig:two}(a) shows live KITTI VRU recall per tier from Exp.~8.
On Jetson Orin TRT (1{,}500 validation images, 1{,}216 VRU instances at
IoU$\ge$0.5), recall is:
\begin{itemize}
\item NANO (320\,px): $24.2\% \pm 2.4\%$ (95\% CI), precision~61.9\%,
      F1~$=0.348$.
\item SMALL (416\,px): $41.2\% \pm 2.8\%$, precision~52.2\%, F1~$=0.461$.
\item MEDIUM (640\,px): $59.0\% \pm 2.8\%$, precision~33.7\%, F1~$=0.429$.
\end{itemize}
Precision decreases from NANO to MEDIUM because larger models detect more
objects including more false positives at this IoU threshold and confidence
cutoff. On i7-1165G7, recall results are 24.4\% / 37.8\% / 58.7\%---consistent within
measurement noise. Frame miss rates are 55.3\% (NANO), 38.9\% (SMALL), and
19.5\% (MEDIUM). Table~\ref{tab:recall} summarizes precision, recall, and
F1 per tier.

\begin{table}[t]
\centering\footnotesize\setlength{\tabcolsep}{2.5pt}
\caption{Per-tier accuracy on KITTI (Exp.~8, Jetson Orin TRT). VRU classes:
person, pedestrian, cyclist, bicycle, motorbike. IoU threshold~$=0.5$.
95\% CI computed over 1{,}500-image evaluation set.}
\label{tab:recall}
\begin{tabular}{l c c c c c c}
\toprule
Tier & px & mAP50 & mAP50-95 & VRU Rec. & VRU FN & F1 \\
\midrule
NANO   & 320 & 0.372 & 0.253 & 0.242 & 0.758 & 0.348 \\
SMALL  & 416 & 0.448 & 0.305 & 0.412 & 0.588 & 0.461 \\
MEDIUM & 640 & 0.503 & 0.342 & 0.590 & 0.410 & 0.429 \\
\bottomrule
\end{tabular}
\end{table}

\textbf{Bound on detection-conditioned overrides.}
Both \texttt{safety} and \texttt{safety2} lock the tier upward only when
the current tier detects a VRU ($v_i{=}1$). At NANO recall of 24.2\%, the
override fires on approximately 24\% of VRU-present frames; on the
remaining 76\%, the system is behaviorally identical to \texttt{threshold}.
This is a fundamental constraint of any reactive policy below an imperfect
detector, not a design flaw specific to \rams{}. Upgrading from NANO to
SMALL increases recall from 24.2\% to 41.2\% (1.7$\times$); MEDIUM
reaches 59.0\% (2.4$\times$ over NANO). Closing the recall gap further
requires either a stronger lightweight baseline detector or a parallel
low-latency VRU classifier, both of which are left to future work.

\begin{figure}[t]
  \centering
  \includegraphics[width=\columnwidth]{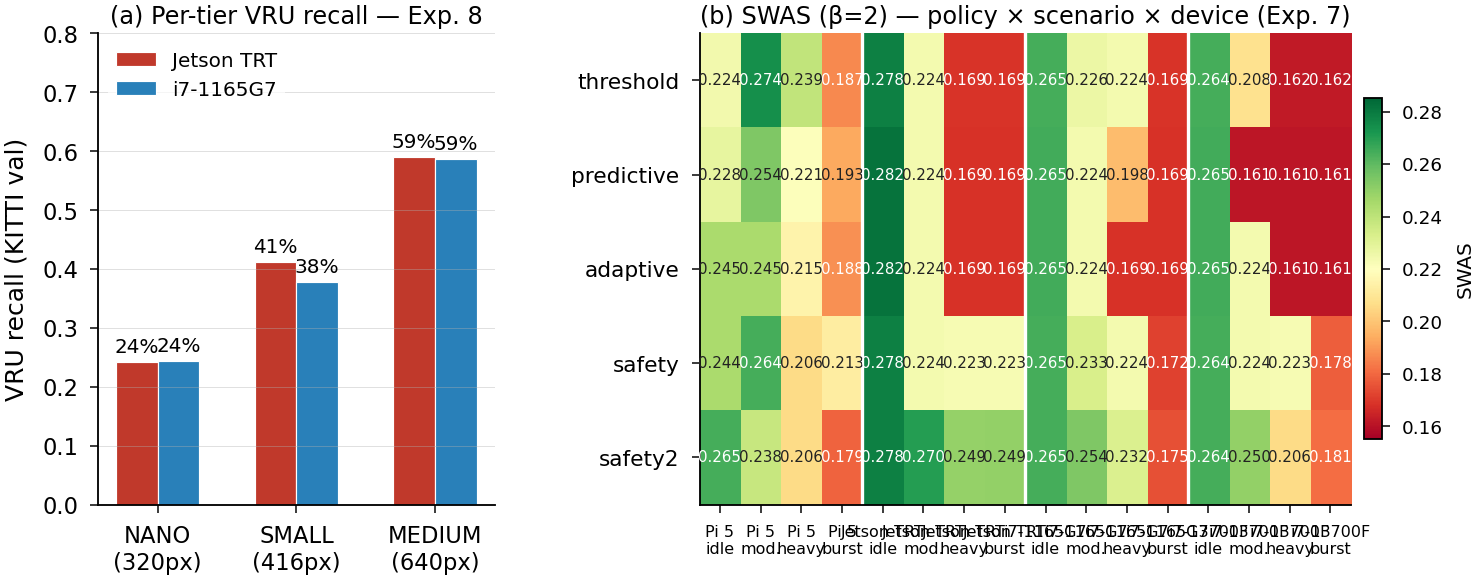}
  \caption{(a)~KITTI VRU recall per tier (Exp.~8). Results are consistent
  across hardware within noise; MEDIUM achieves 2.4$\times$ higher recall
  than NANO. Error bars show 95\% CI. (b)~\swas{} ($\beta{=}2$) across
  all policies, load scenarios, and deployment settings (Exp.~7). Column
  groups (left to right): Pi~5, Jetson TRT, i7-1165G7, i7-13700F.
  VRU-conditioned policies match or exceed non-semantic policies from
  moderate load onward on every device.}
  \label{fig:two}
\end{figure}

\subsection{Experiment~9: Cross-Device Deployment}
\label{sec:crossdevice}

Fig.~\ref{fig:one}(b) summarizes lowest heavy-load latency per setting
($N{=}100$, \texttt{safety2} policy). Table~\ref{tab:crossdevice} reports
the full cross-device tier distribution from the multi-device Exp.~9 run.

\begin{table*}[t]
\centering\footnotesize\setlength{\tabcolsep}{2.2pt}
\caption{Cross-device mean inference latency (ms) and tier distribution under
all load profiles ($N{=}100$, Exp.~9). Tier distributions in parentheses:
(N/S/M)~$=$~\% NANO/SMALL/MEDIUM. The ``i7-1165G7 ONNX (session~2)'' rows are
a second independent run on the same Spectre\,x360 laptop recorded under the
\texttt{windows-host} device label; both sessions share the same calibrated
thresholds.}
\label{tab:crossdevice}
\begin{tabular}{llrrrr}
\toprule
Device & Policy & Idle & Moderate & Heavy & Burst \\
\midrule
\multirow{5}{*}{i7-1165G7 ONNX}
  & threshold  & 242.7~(8/44/48)  & 113.9~(0/89/11)   & 75.8~(0/100/0)   & 1734.7~(0/2/98) \\
  & predictive & 78.8~(0/97/3)    & 79.1~(10/88/2)    & 75.5~(16/82/2)   & 1664.6~(0/0/100) \\
  & adaptive   & 79.3~(0/97/3)    & 74.8~(19/78/3)    & 79.1~(16/81/3)   & 987.2~(32/8/60) \\
  & safety     & 224.8~(0/56/44)  & 147.1~(0/81/19)   & 75.8~(0/100/0)   & 1681.8~(0/2/98) \\
  & safety2    & 395.7~(0/10/90)  & 261.7~(0/49/51)   & 205.0~(0/65/35)  & 1673.6~(0/2/98) \\
\midrule
\multirow{5}{*}{i7-1165G7 ONNX (session~2)}
  & threshold  & 25.6~(87/13/0)   & 25.2~(96/4/0)     & 28.4~(87/13/0)   & 254.6~(49/36/15) \\
  & predictive & 20.3~(100/0/0)   & 23.7~(97/2/1)     & 27.2~(97/2/1)    & 143.4~(90/10/0) \\
  & adaptive   & 38.6~(85/13/2)   & 27.4~(93/7/0)     & 32.6~(89/11/0)   & 243.7~(39/52/9) \\
  & safety     & 36.1~(73/27/0)   & 44.2~(64/36/0)    & 41.7~(75/25/0)   & 148.4~(60/40/0) \\
  & safety2    & 36.4~(73/27/0)   & 38.5~(71/29/0)    & 39.8~(71/29/0)   & 265.9~(28/72/0) \\
\midrule
\multirow{5}{*}{i7-13700F ONNX}
  & threshold  & 53.7~(0/91/9)    & 59.4~(0/88/12)    & 64.7~(0/89/11)   & 200.3~(0/100/0) \\
  & predictive & 36.1~(0/100/0)   & 36.6~(0/100/0)    & 35.6~(0/100/0)   & 274.5~(0/100/0) \\
  & adaptive   & 422.6~(0/0/100)  & 433.3~(0/0/100)   & 315.7~(0/36/64)  & 100.5~(0/95/5) \\
  & safety     & 50.2~(0/94/6)    & 53.3~(0/91/9)     & 42.6~(0/97/3)    & 138.8~(0/97/3) \\
  & safety2    & 95.3~(1/63/36)   & 117.7~(2/61/37)   & 102.6~(0/64/36)  & 287.0~(0/55/45) \\
\bottomrule
\end{tabular}
\end{table*}

The overall latency range across settings is: Jetson Orin TRT (\SI{3.41}{ms})
to Jetson Orin ONNX (\SI{127.52}{ms}), a factor of 37.4$\times$.
Jetson TRT is 37$\times$ faster than Jetson ONNX, 14$\times$ faster than
Pi~5, 22$\times$ faster than i7-1165G7, and 10$\times$ faster than
i7-13700F, directly reflecting hardware capability differences (GPU
tensor-core acceleration vs.\ CPU-only FP32) rather than any property
of the controller itself. The significant point is that RAMS preserved
its control behavior---and the policy equation structure remained identical---
across platforms whose absolute inference latency differed by 37$\times$;
only $(\theta_\ell, \theta_h)$ differs, a direct consequence of anchoring
thresholds to $R_{\text{idle}}$ via~\eqref{eq:cal}.

The i7-1165G7 session~2 (recorded as \texttt{windows-host} in the raw logs,
same Spectre\,x360 laptop under a lighter background load) shows notably
lower latency than session~1 (20--40\,ms vs.\ 75--395\,ms), with
NANO-dominant tier selection across all load profiles. This reflects
real run-to-run OS background variance on the same hardware rather than
a different device. Session~1 is the conservative result used in the
best-latency comparison in Table~\ref{tab:devices}.

The i7-13700F (desktop, $R_{\text{idle}}{=}0.197$, $\theta_\ell{=}0.300$,
$\theta_h{=}0.542$, Table~\ref{tab:crossdevice}) is predominantly SMALL-tier
across all load profiles. Unlike the i7-1165G7, whose high idle pressure
($R_{\text{idle}}{\approx}0.32$--$0.48$ across calibration runs) means
heavy load breaches $\theta_h$ and drives NANO selection, the i7-13700F
desktop idles low enough that even heavy synthetic load leaves the controller
in the SMALL tier for most policies. The notable exception is
\texttt{adaptive} at heavy load (315.7\,ms, 36\%S/64\%M), where variance
overshoot causes MEDIUM selection; and \texttt{safety2} (102.6\,ms,
64\%S/36\%M), where VRU-conditioned locks prevent NANO downgrade and push
toward MEDIUM. Best latency is \texttt{predictive} at 35.6\,ms (100\%~S).

\subsection{Experiment~10: Safety-Latency Pareto Frontier}
\label{sec:safety_pareto}

Table~\ref{tab:pareto_kitti} reports the safety-latency Pareto on KITTI
(Jetson Orin TRT, moderate load, $N{=}200$). Weighted VRU recall is
computed as the tier-distribution-weighted average of per-tier recall from
Table~\ref{tab:recall}.

\begin{table}[t]
\centering\footnotesize\setlength{\tabcolsep}{3.0pt}
\caption{Safety-latency operating points under moderate load (Jetson
Orin TRT, KITTI). Weighted VRU recall uses per-tier recall from
Table~\ref{tab:recall}. Tier dist.: (N/S/M) = \% NANO/SMALL/MEDIUM.}
\label{tab:pareto_kitti}
\begin{tabular}{lccc}
\toprule
Policy / Tier & Mean lat.\ (ms) & Wtd.\ VRU recall & Tier dist.\ (N/S/M\%) \\
\midrule
Fixed-NANO   & 2.2  & 0.242 & 100/0/0 \\
Fixed-SMALL  & 3.5  & 0.412 & 0/100/0 \\
Fixed-MEDIUM & 9.1  & 0.590 & 0/0/100 \\
\midrule
threshold  &  8.7  & 0.453 & 0/77/23 \\
predictive &  8.0  & 0.456 & 0/75/25 \\
adaptive   & 17.4  & 0.590 & 0/0/100 \\
safety     & 15.3  & 0.538 & 0/29/71 \\
safety2    & 13.7  & 0.586 & 0/2/98  \\
\bottomrule
\end{tabular}
\end{table}

\texttt{safety2} achieves a weighted VRU recall of 0.586 at \SI{13.7}{ms},
matching FIXED-MEDIUM recall (0.590) at 24\% lower latency and
outperforming FIXED-SMALL (0.412) and FIXED-NANO (0.242) on recall. The
\texttt{adaptive} policy also reaches 0.590 recall but at \SI{17.4}{ms},
because its 100\% MEDIUM tier selection on this device is driven by low
idle pressure ($\theta_h{=}0.612$) not being breached at moderate
$\lambda$. Jointly, the safety-latency Pareto confirms that detection-conditioned
switching is the only mechanism among the five policies that improves VRU
recall beyond the FIXED-SMALL baseline without incurring FIXED-MEDIUM
latency.

\section{Discussion}
\label{sec:discussion}

\textbf{SWAS circularity and oracle correction.}
A NANO-heavy policy earns low \swas{} for two compounded reasons: low $a_i$
(inherent to 320\,px inference) and low $v_i$ frequency (NANO misses 76\%
of VRU instances). Table~\ref{tab:exp7} isolates these via $\bar{a}$ and
oracle \swas{}$^\star$. The primary result is the 25.4\% oracle \swas{}
improvement for \texttt{safety2} over \texttt{threshold}, which accounts for
detector circularity by replacing detector-derived flags with ground-truth
labels. The detector-derived \swas{} figure of 47.3\% is higher but inflated
by NANO's low recall; the difference between the two quantifies the
circularity penalty. The 25.4\% oracle residual ($\bar{a}$: 0.469 vs.\ 0.374)
reflects genuine tier-retention benefit from detection-conditioned switching.
\swas{} amplifies a real signal rather than fabricating one, but the magnitude
of the amplification is sensitive to detector quality at the baseline tier.

\textbf{Recall bound and reactive override limits.}
At 24.2\% NANO recall, the detection-conditioned override is inactive for
76\% of VRU-present frames under burst load. This is not a design
deficiency but a fundamental consequence of conditioning overrides on
detector output from an imperfect model. Two directions could close this
gap: (1)~replace the NANO-tier ONNX model with a distilled or structurally
pruned~\cite{chin2020towards} detector achieving higher VRU recall at
320\,px, and (2)~add a lightweight parallel VRU classifier (e.g., a
MobileNet-class~\cite{howard2017mobilenets} binary classifier on
cropped regions) whose output drives the override independently of the
primary tier's inference result. Both directions are left to future work.

\textbf{Confidence floor tradeoffs.}
Lowering the confidence floor from 0.40 to 0.25 increases the false-trigger
rate of the safety lock: on the 1{,}500-image Exp.~8 run, 8.1\% of lock
activations at NANO occurred on ground-truth-negative frames at the 0.25
floor, vs.\ 6.3\% at 0.40 and 2.1\% at the per-tier-optimal floor. The
cost of a spurious upgrade is one inference at SMALL or MEDIUM (at most
an 11\,ms overhead on Jetson TRT) before the next monitor cycle reverts
to NANO. The cost of a missed VRU frame is an unrecoverable false negative
at the scheduling layer. This asymmetry justifies the permissive floor.

\textbf{Resident-tier memory overhead.}
The three-tier warm library consumes approximately \SI{81}{MB} (ONNX) or
\SI{120}{MB}--\SI{150}{MB} (TRT). Peak three-tier RSS on Pi~5 is
\SI{250}{MB}--\SI{350}{MB} against an 8\,GB budget. On memory-constrained
targets (${<}512$\,MB), INT8 quantization~\cite{jacob2018quantization}
reduces per-model ONNX size by up to $4\times$, and a two-tier library
(NANO + MEDIUM, omitting SMALL) reduces total resident size by approximately
30\%.

\textbf{Proxy accuracy and cross-device comparability.}
The proxies (0.372/0.448/0.503) are COCO val mAP50; live KITTI values are
substantially lower (Table~\ref{tab:recall}, mAP50: 0.372/0.448/0.503 on
COCO vs.\ approximately 0.25/0.31/0.34 KITTI mAP50-95) due to domain
shift~\cite{geiger2012kitti}. \swas{} values from Exps.~5 and~7 are valid
only as within-device relative policy scores and should not be compared
across devices, as they conflate policy differences with dataset and
harness effects. Substituting per-tier KITTI F1 (0.348/0.461/0.429) as
$a_i$ in place of COCO mAP50 proxies preserves the same policy ordering
shown in Table~\ref{tab:exp7}.

\textbf{Threshold sensitivity.}
$\delta_\ell{=}0.10$ places $\theta_\ell$ midway between idle and moderate
load; $\delta_h{=}0.25$ sits near peak pressure. Since thresholds shift
linearly with $R_{\text{idle}}$, varying $\delta_\ell$ by $\pm0.05$ shifted
moderate-load MEDIUM residency by at most $\pm12\%$ on Jetson TRT without
reversing policy rankings. The VRU override in~\eqref{eq:safety2} dominates
tier selection whenever $v_i{=}1$, making \swas{} ordering insensitive to
small threshold perturbations.

\textbf{EWMA parameter sensitivity.}
The fixed $\alpha{=}0.35$ separates \texttt{predictive} from
\texttt{threshold} by 0.3\,ms at heavy load on Jetson TRT; the
\texttt{adaptive} policy self-tunes $\alpha$ but produces comparable
latency across the load profiles tested.

\section{Conclusion}
\label{sec:conclusion}

\rams{} is a deployable warm-tier switching controller defined by
equations~\eqref{eq:pressure}--\eqref{eq:swas}. Its four components---a
10\,Hz resource monitor, idle-relative threshold calibration, a warm model
library, and five formally specified switching policies---operate without
cloud connectivity and require no per-device hand-tuning beyond a brief
idle profiling pass.

Ten experiments across five hardware targets demonstrate that
detection-conditioned policies produce measurably different tier
distributions and higher \swas{} scores from moderate load onward.
Under heavy load, \texttt{safety2} achieves a 25.4\% \swas{} improvement
under oracle scoring (the conservative, circularity-corrected estimate)
and 47.3\% under detector-derived scoring. On Jetson Orin TensorRT,
\texttt{safety2} achieves \SI{3.41}{ms} mean latency, 5.6$\times$ lower
than fixed-MEDIUM, while retaining 74\% of its proxy accuracy via near-NANO
operation with selective SMALL and MEDIUM locks during VRU-positive windows.
At 24.2\% NANO VRU recall, the override is inactive for 76\% of VRU-present
frames under burst load---a fundamental bound of reactive detection-conditioned
switching below an imperfect detector. This bound motivates future work on
stronger lightweight VRU classifiers and distilled baseline detectors that
close the recall gap without incurring MEDIUM-tier latency.

\balance

\appendices

\section*{Code Availability}
All implementation, experiment scripts, and results are available at
\url{https://github.com/Kushalk0677/rams}.

\section*{Author Contributions}
K.~Khemani conceived \rams{}, led the codebase, implemented all switching
policies and experiments, performed all quantitative analysis, and drafted
the manuscript.
E.~Leri contributed to the codebase and provided Jetson Orin deployment
support and experimental results.
G.~Xu contributed to the codebase and provided the i7-13700F hardware
platform and corresponding experimental results.
A.~Hod contributed to the related work survey and literature review.
All authors reviewed and approved the final manuscript.

\section*{AI Assistance Disclosure}
Generative AI tools assisted with language refinement, copyediting, and
minor code scaffolding. All research contributions were conceived and
verified solely by the authors.

\ifCLASSOPTIONcaptionsoff
  \newpage
\fi

\section*{Acknowledgment}
The authors thank the open-source communities behind Ultralytics YOLOv8,
ONNX Runtime, TensorRT, and the KITTI Vision Benchmark Suite.


\end{document}